%% file: neurips_2025.tex
\documentclass{article}

    \usepackage[preprint]{neurips_2025}

\usepackage{subcaption}
\usepackage[utf8]{inputenc} 
\usepackage[T1]{fontenc}    
\usepackage{hyperref}       
\usepackage{url}            
\usepackage{booktabs}       
\usepackage{amsfonts}       
\usepackage{amsmath}        
\usepackage{nicefrac}       
\usepackage{microtype}      
\usepackage{xcolor}         
\usepackage{microtype}
\usepackage{graphicx}
\usepackage{diagbox}
\usepackage{booktabs} 

\usepackage{algorithm}
\usepackage{algorithmic}
\newtheorem{theorem}{Theorem}[section]
\newtheorem{proof}{Proof}[section]

\usepackage{tikz}
\usetikzlibrary{positioning, arrows.meta, shapes.geometric, calc}

\title{Policy-Based Trajectory Clustering in Offline Reinforcement Learning}

\author{
  Hao Hu\thanks{Work done while Hao Hu was visiting the University of Washington.} ~$^\dagger$\\
  Institute for Interdisciplinary Information\\ Sciences \\
  Tsinghua University \\
  \texttt{h-hu22@mails.tsinghua.edu.cn}
  \And 
  Xinqi Wang\thanks{Equal contribution.} \\
  Paul G. Allen School of Computer Science\\ \& Engineering \\
  University of Washington \\
  Seattle, WA 98195 \\
  \texttt{wxqkaxdd@cs.washington.edu}
  \AND
  Simon S. Du \\
  Paul G. Allen School of Computer Science\\ \& Engineering \\
  University of Washington \\
  Seattle, WA 98195 \\
  \texttt{ssdu@cs.washington.edu}
}

\begin{document}

\maketitle

\begin{abstract}
    \input{texs/abstract}
\end{abstract}

\input{texs/introduction}
\input{texs/related_work}

\input{texs/preliminaries}
\input{texs/comparison}
\input{texs/method}

\input{texs/caae}
\input{texs/experiments}

\begin{ack}
SSD acknowledges the support of NSF DMS 2134106, NSF CCF 2212261, NSF IIS 2143493, NSF IIS 2229881, Alfred P. Sloan Research Fellowship, and Schmidt Sciences AI 2050 Fellowship.
\end{ack}

\bibliography{neurips_2025}
\bibliographystyle{plainnat}

%
%
%
%
%
%
%
%

\newpage
\appendix
\tableofcontents

\input{texs/app_proofs}
\input{texs/app_experiments}



\end{document}

%% file: texs/abstract.tex
We introduce a novel task of clustering trajectories from offline reinforcement learning (RL) datasets, where each cluster center represents the policy that generated its trajectories. By leveraging the connection between the KL-divergence of offline trajectory distributions and a mixture of policy-induced distributions, we formulate a natural clustering objective. To solve this, we propose Policy-Guided K-means (PG-Kmeans) and Centroid-Attracted Autoencoder (CAAE). PG-Kmeans iteratively trains behavior cloning (BC) policies and assigns trajectories based on policy generation probabilities, while CAAE resembles the VQ-VAE framework by guiding the latent representations of trajectories toward the vicinity of specific codebook entries to achieve clustering.
Theoretically, we prove the finite-step convergence of PG-Kmeans and identify a key challenge in offline trajectory clustering: the inherent ambiguity of optimal solutions due to policy-induced conflicts, which can result in multiple equally valid but structurally distinct clusterings.
Experimentally, we validate our methods on the widely used D4RL dataset and custom GridWorld environments. Our results show that both PG-Kmeans and CAAE effectively partition trajectories into meaningful clusters. They offer a promising framework for policy-based trajectory clustering, with broad applications in offline RL and beyond.

%% file: texs/introduction.tex
\section{Introduction}

In recent years, reinforcement learning (RL) has achieved significant progress across a wide range of domains, including robotic control \cite{tang2024deepreinforcementlearningrobotics}, autonomous driving \cite{kiran2021deepreinforcementlearningautonomous}, and recommendation systems \cite{Lin_2024}. However, conventional online RL methods typically rely on continuous interactions with the environment to explore and optimize policies. In many real-world scenarios, such frequent interactions are not only expensive but also pose considerable safety risks—especially in sensitive applications like medical diagnosis and autonomous driving \cite{dulacarnold2019challengesrealworldreinforcementlearning}.

To overcome these limitations, offline reinforcement learning (Offline RL) has emerged as a promising alternative. It aims to learn optimal policies from fixed, pre-collected datasets without further interaction with the environment. The success of this paradigm depends heavily on the quality, coverage, and structure of the offline data \cite{levine20orlsurvey}. As a result, efficiently organizing and utilizing offline datasets has become a core challenge in offline RL research.

Existing offline RL approaches can be broadly categorized into two paradigms based on their use of offline data. The first, Q-learning-based methods, tackle the issue of distributional shift by constraining the learned policy to remain close to the behavior policy \cite{fujimoto2019offpolicydeepreinforcementlearning,kumar2019stabilizingoffpolicyqlearningbootstrapping,wu2019behaviorregularizedofflinereinforcement}, or by leveraging conservative value estimation and uncertainty-aware regularization \cite{kumar20cql,yu2021combo,janner19mbpo,kidambi2021morelmodelbasedoffline}.

The second category, behavior cloning-based methods, directly optimize policies using supervised learning on offline trajectories \cite{chen2020bailbestactionimitationlearning}, or impose constraints on policy updates to avoid distributional shift \cite{brandfonbrener2021offlinerloffpolicyevaluation,kostrikov2022offline}. Some approaches further refine policies through techniques such as importance sampling or trajectory optimization \cite{zhang20gendice,nachum2019algaedicepolicygradientarbitrary,nachum19dice,janner2021offlinereinforcementlearningbig}.

In this work, we propose a novel perspective on leveraging offline data. Offline datasets are often generated by a mixture of distinct policies, resulting in diverse behavioral patterns. Training a single policy on such heterogeneous data can lead to interference among these patterns. Even when the dataset comprises high-quality policies, the presence of varied behavioral styles may increase training complexity, leading to instability and degraded performance. To address this issue, we propose clustering offline RL trajectories to better capture and exploit the underlying structure of behavior policies.
Below, we outline the key motivations for this approach:
\begin{enumerate}
    \item Improving offline RL training and enhancing data utilization.
Offline RL often suffers from distributional shift, especially when learning from suboptimal or heterogeneous datasets. By clustering data according to the underlying behavior policies, we can selectively exploit high-quality trajectories, filter out misleading or low-quality samples, and train policies that generalize more effectively and robustly.
    \item Enabling better understanding and evaluation of RL policies.  
Clustering provides a principled way to analyze and compare different behavior patterns across policies. This structural insight facilitates debugging, enhances interpretability, and guides the development of more effective training strategies.

    \item Facilitating semi-supervised offline RL with minimal reward information.  
    In many real-world scenarios, reward annotations are expensive or infeasible to obtain. A promising direction is to first organize the dataset through unsupervised clustering, and then guide policy learning using a small subset of reward-labeled data. This semi-supervised learning paradigm reduces dependence on dense reward supervision while maintaining strong performance, making it especially appealing in practical applications.

\end{enumerate}


\textbf{Main Contributions.} This paper provides a systematic study of policy-based trajectory clustering. Our key contributions are summarized as follows:

\begin{enumerate}  
    \item We formulate and define the problem of policy-based trajectory clustering, laying the foundation for this new research direction. See Section~\ref{sec:prelims}.

    \item We compare policy-based trajectory clustering with traditional clustering methods from the perspectives of data characteristics, methodology, and solution uniqueness.  
    In particular, we establish a theoretical connection to the \emph{K-coloring problem}, demonstrating a fundamental difference between policy-based trajectory clustering and classical clustering problems in the worst-case scenario.\footnote{Recall that the clustering solution under the $L_2$-distance criterion is unique.} See Section~\ref{sec:Policy Clustering vs Trad clustering}.

    \item Inspired by the classical K-means algorithm and probabilistic mixture models, we propose two new algorithm, Policy-Guided K-means (PG-Kmeans) and Centroid-Attracted Autoencoder(CAAE).  
    These two approaches adopt fundamentally different design philosophies. PG-Kmeans explicitly maintains distinct central policies for each cluster, resulting in clear policy segmentation and enabling direct one-step policy prediction. In contrast, CAAE performs dataset-level clustering using a unified encoder-decoder architecture, which requires fewer computational resources for training and demonstrates more stable performance across experiments.
 
    
    \item We conduct extensive experiments on the D4RL Gym benchmark and newly built environments, comparing our methods with other plausible clustering methods.  
    Experimental results show that PG-Kmeans and CAAE achieves more effective policy-based trajectory clustering under the standard Normalized Mutual Information (NMI) metric. See Section~\ref{sec: experiments}.
\end{enumerate}  

%% file: texs/related_work.tex
\section{Related Work}

\paragraph{Offline Reinforcement Learning.}
A central challenge in offline reinforcement learning (RL) is mitigating the distributional shift between the learned policy and the behavior policy from which the dataset was collected. Existing methods can be broadly grouped into three main categories. The first class constrains the learned policy to remain close to the behavior policy, either explicitly through regularization \cite{fujimoto2019offpolicydeepreinforcementlearning,kumar2019stabilizingoffpolicyqlearningbootstrapping,wu2019behaviorregularizedofflinereinforcement}, or implicitly via conservative value estimation to prevent overestimation of rewards \cite{kumar20cql,yu2021combo}. The second class focuses on uncertainty estimation, often employing ensemble methods to penalize unreliable actions and enhance robustness in out-of-distribution regions \cite{janner19mbpo,kidambi2021morelmodelbasedoffline}. The third class includes behavior cloning-based approaches, which either treat behavior cloning as a surrogate for policy learning \cite{chen2020bailbestactionimitationlearning}, or constrain policy updates to one-step improvements to avoid error accumulation in off-policy evaluation \cite{brandfonbrener2021offlinerloffpolicyevaluation,kostrikov2022offline}. In parallel, alternative strategies such as importance sampling and trajectory reweighting aim to directly optimize over trajectory distributions without explicitly correcting for distributional shift \cite{zhang20gendice,nachum2019algaedicepolicygradientarbitrary,janner2021offlinereinforcementlearningbig}. Despite their methodological differences, all these approaches critically rely on the coverage, quality, and structure of the offline dataset.

\paragraph{Policy-Based Clustering.}
Policy-based clustering has recently emerged as a promising direction for decomposing complex, heterogeneous offline RL datasets into more interpretable and manageable behavior modes. SORL \cite{mao2024stylized} introduces an expectation-maximization framework that alternates between trajectory clustering and policy optimization, enabling the discovery of diverse and high-quality behaviors. Similarly, \citet{wang2024datasetclusteringimprovedoffline} propose behavior-aware deep clustering to isolate uni-modal behavioral subsets from multi-behavioral datasets, thereby improving policy stability and performance. Other approaches use probabilistic models for implicit clustering. For instance, \citet{li2023offlinereinforcementlearningclosedform} utilize Gaussian Mixture Models (GMMs) to represent latent behavior policies and derive closed-form policy improvement operators. Diffusion-QL \cite{wang2023diffusion} leverages expressive diffusion models to model the multimodal distribution of behavior policies. Collectively, these works highlight the value of policy-level clustering in improving both generalization and robustness in offline RL.

\paragraph{VAEs in Reinforcement Learning.}
Variational Autoencoders (VAEs) \cite{kingma2013vae} are widely used in reinforcement learning for unsupervised representation learning, particularly in high-dimensional or visual domains. They have been employed in model-based RL \cite{ha2018world}, hierarchical skill discovery \cite{pertsch2020spirl}, and curiosity-driven exploration \cite{mohamed2015variational, klissarov2019variational}. A common baseline for trajectory-level modeling is to embed trajectories into latent spaces using VAEs, followed by clustering with standard algorithms like K-means. While simple and computationally efficient, this strategy has key limitations: the learned embeddings may not retain policy-level semantics, and distance-based clustering methods such as K-means are ill-suited for the temporally structured nature of decision-making data (see Table~\ref{tab:NMI}). These limitations motivate the design of our CAAE model, which explicitly targets policy-consistent representation learning with clustering objectives in mind.

\paragraph{Deep Clustering.}
Deep clustering methods aim to jointly learn representations and perform clustering in a unified, often end-to-end, framework. A typical approach involves first learning low-dimensional embeddings using neural networks, then applying clustering algorithms such as K-means or GMMs. DEC \cite{xie2016unsupervised} improves cluster assignments via iterative refinement in an autoencoder architecture, while DEPICT \cite{dizaji2017deep} introduces convolutional backbones for improved image clustering. DAC \cite{chang2017deep} enforces pairwise similarity constraints to enhance representation learning. More recent advancements explore contrastive learning \cite{li2021contrastive}, graph-based methods that exploit structural relationships \cite{bo2020structural}, and fully end-to-end clustering models \cite{ji2019invariant}. These advances provide powerful tools for unsupervised data structuring, many of which can be adapted or extended to sequential decision-making data in RL.

%% file: texs/preliminaries.tex
\section{Problem Setup: Policy-Based Trajectory Clustering} 
\label{sec:prelims}

We consider a dataset clustering problem in finite-horizon Markov Decision Process (MDP) offline reinforcement learning (Offline RL), where the objective is to cluster given trajectories based on the policy that generated them.

An MDP is defined as $(\mathcal{S}, \mathcal{A}, P, r, H)$, where $\mathcal{S}$ and $\mathcal{A}$ denote the state and action spaces, respectively. The transition dynamics are given by $P: \mathcal{S} \times \mathcal{A} \to \Delta(\mathcal{S})$, the reward function is $r:\mathcal{S} \times \mathcal{A}\to \mathbb{R}$, and $H$ represents the horizon. A trajectory $\tau$ is defined as $\{s_t, a_t, r_t\}_{t=1}^H$ in the full-reward setting or $\{s_t, a_t\}_{t=1}^H$ in the reward-free setting, where $s_t \in \mathcal{S}$ and $a_t \in \mathcal{A}$. A deterministic policy $\pi: \mathcal{S} \times [H] \to \mathcal{A}$ maps each state-timestep pair to an action.

We assume the dataset consists of a mixture of $k$ sub-datasets, each collected under a different deterministic behavior policy $\{\pi_i\}_{i=1}^k$. Formally,  
\[
\mathcal{D} = \bigcup_{i=1}^k \mathcal{D}_i,
\]
where each $\mathcal{D}_i$ is collected by repeating playing$\pi_i$. The objective is to cluster the trajectories according to their source. A alternative objective, considering it as a K-coloring problem, is provided in Appendix~\ref{app: k-coloring}.

For evaluation, we denote ground-truth labels and predicted labels by $L$ and $C$, respectively, and use Normalized Mutual Information (NMI) \cite{strehl2002cluster} as the performance metric:
\[
\text{NMI}(C, L) = \frac{2 \cdot I(C, L)}{H(C) + H(L)},
\]
where \(I(C, L)\) is the mutual information between \(C\) and \(L\), and \(H(\cdot)\) denotes Shannon entropy.

%% file: texs/comparison.tex
\section{Policy-Based Trajectory Clustering vs. Traditional Clustering}
\label{sec:Policy Clustering vs Trad clustering}

\paragraph{Nature of Data.}
Traditional deep clustering techniques are typically designed for static data domains such as images, text, or audio. In these settings, individual samples can be embedded into a high-dimensional feature space, and clustering is performed based on similarity in that space. In contrast, offline reinforcement learning datasets comprise sequential state-action pairs generated by latent policies. Here, similarity is not merely a function of geometric proximity in feature space, but also of the underlying generative process, i.e., the decision-making policy.

Due to the stochasticity and high-dimensionality of RL environments, two trajectories that appear similar in state space may in fact arise from fundamentally different policies. As a result, directly applying conventional clustering methods (e.g., K-means or GMM) to raw or embedded state-action sequences often leads to clusters that reflect environmental or task-level structure rather than true policy-level consistency. This misalignment compromises both interpretability and the utility of the clusters for downstream policy learning.

\paragraph{Ambiguity in Policy-Based Trajectory Clustering: A K-Coloring Perspective}
\label{subsec: Ambiguity in Policy Clustering}

Unlike classical clustering tasks, policy-based trajectory clustering may not admit a unique solution. To illustrate this, we draw a reduction to the well-known \textit{K-coloring problem}. Specifically, we construct a graph where each node represents a trajectory, and an edge connects two nodes if the corresponding trajectories exhibit conflicting decision behavior—i.e., they cannot have been generated by the same stationary policy. In this formulation, a valid K-coloring corresponds to a feasible partition of trajectories into policy-consistent clusters.

We further provide an inverse reduction (see Appendix~\ref{app: k-coloring}) and show that the policy-based trajectory clustering problem is NP-complete. This theoretical result highlights a core distinction from standard clustering: the existence of multiple, equally valid clusterings. While this introduces inherent ambiguity, empirical evidence suggests that stable and semantically meaningful solutions can still be obtained under realistic assumptions on the data distribution (see Appendix~\ref{app: policy ambiguity}).


%% file: texs/method.tex
\section{Methods}
\label{sec: methods}

To address the problem of policy-based trajectory clustering, we consider two general approaches.

First, we propose Policy-Guided KMeans (PG-Kmeans), which explicitly maintains $k$ distinct clusters along with their corresponding policy centroids for clustering. This method operates in the space of behaviors by aligning trajectories with representative policies.

Alternatively, in the direction of representation learning, we introduce the Centroid-Attracted Autoencoder (CAAE), which trains an encoder to map trajectories into a low-dimensional latent space. Clustering is then performed in this latent space based on the distance between the embedded representation and a set of learnable centroids.

\subsection{Policy-Guided Kmeans}
\label{sec: PG-Kmeans}

The algorithmic foundation of PG-Kmeans is similar to that of the standard K-means clustering algorithm. In K-means-like algorithms, the clustering objective is typically defined as the sum of distances between data points and their respective cluster centers, which is optimized using the Expectation-Maximization (EM) method.

We derive our clustering objective from an information-theoretic perspective, minimizing the \textit{Kullback-Leibler (KL) divergence} between the empirical data distribution \( P(\tau) \) and the mixture of policy-induced distributions \( \hat{P}(\tau) \), where 
$$
\hat{P}(\tau) = \sum_{j=1}^{K} w_{i,j} \mathbb{P}(\tau | \theta_j).
$$
Here, $\mathbf{W}$ = \( \{w_{i,j}\}_{(i,j)=(1,1)}^{(N,K)} \) represents cluster assignments, \( N \) is the number of data points, and \( K \) is the number of clusters. The KL divergence quantifies the discrepancy between these distributions:

\begin{equation}
D_{\mathrm{KL}}( P(\tau) \| \hat{P}(\tau) ) = \mathbb{E}_{\tau\sim P} \left[\log \frac{P(\tau)}{\hat{P}(\tau)}\right].
\end{equation}

Minimizing this divergence ensures that the learned policies \( \{\pi_j\}_{j=1}^K \) effectively approximate the empirical data distribution. Since $P(\tau)$ is independent of the optimization variables \( \theta_j \) and \( w_{i,j} \), minimizing KL divergence is equivalent to maximizing:

\begin{equation}
\mathbb{E}_{\tau\sim P} \left[\log \hat{P}(\tau)\right] \approx \sum_{i=1}^{N}\log \hat{P}(\tau_i)  = \sum_{i=1}^{N} \sum_{j=1}^{K} w_{i,j} \log \mathbb{P}(\tau_i | \theta_j). 
\end{equation}

Here we denote $\tau_i$ as the sequence $\{(s_{i,h}, a_{i,h})\}_{h=1}^H$. Noting that the assignment weights are binary (\(0\) or \(1\)) and that the environment dynamic is independent of $\theta_j$ and $w_{i,j}$, we arrive at the final objective function:

\begin{align}
    \underset{\mathbf{\theta},\mathbf{W}}{\mathrm{maximize }} \, J(W, \mathbf{\theta}) &= \sum_{i=1}^N \sum_{j=1}^K w_{i,j} \sum_{t=1}^H \log \mathbb{P}(a_{i,h}|\theta_j, s_{i,h}), \label{equ: objective function}\\
    \mathrm{s.t.} \quad w_{i,j} &\in \{0, 1\}, \, \forall (i,j) \in [N] \times [K], \notag\\
    \sum_{j=1}^K w_{i,j} &= 1, \, \forall i \in [N]. \notag
\end{align}

This objective function formulates an optimal distribution matching problem, where we seek to approximate the empirical data distribution using a mixture of policy-induced distributions. The cluster assignments \( w_{i,j} \) ensure that each data point is assigned to the most suitable policy, while the policy parameters \( \theta_j \) are optimized to maximize the likelihood of the assigned trajectories.
\subsubsection{Main Algorithm}

To solve the clustering objective, we derive an iterative algorithm based on the Expectation-Maximization (EM) framework. See Algorithm~\ref{alg:pg_kmeans} for the pseudocode. Similar to K-means, the algorithm alternates between an E-step, where data points are assigned to the most probable generating cluster, and an M-step, where each cluster center is updated to maximize the likelihood of generating its assigned trajectories. This process iterates until convergence, after which cluster centers may be merged to further optimize the objective function \( J(W, \theta) \).

\begin{algorithm}[ht!]
    \caption{Policy-Guided Kmeans (PG-Kmeans)}
    \label{alg:pg_kmeans}
    \begin{algorithmic}[1]
        \STATE \textbf{Input:} Dataset $\mathcal{D} = \{\tau_1, \dots, \tau_N\}$, number of clusters $k$, ground truth $k^*$ (optional), maximum iterations $T$.
        \STATE \textbf{Initialize} cluster assignments $W$ randomly.
        \STATE \textbf{Initialize} $k$ policies $\{\pi_1, \pi_2, \dots, \pi_k\}$ using behavior cloning (BC) trained on the respective clusters.
        \FOR{$t = 1$ to $T$}
            \STATE \textbf{M-step:} Update each policy $\pi_j$ by training a behavior cloning model on the trajectories assigned to cluster $j$.
            \STATE \textbf{E-step:} Assign each trajectory $\tau_i \in \mathcal{D}$ to the cluster \( j = \arg\max_{j'}\mathbb{P}(\tau_i|\theta_{j'}) \), i.e., \( w_i = e_j \), where \( e_j \) is a one-hot vector.
            \STATE \textbf{Check Convergence:} If cluster assignments \( W \) remain unchanged, terminate.
        \ENDFOR
        \STATE If $k^*$ is given, run Algorithm~\ref{alg:merge} to merge clusters.
        \STATE \textbf{Output:} Final cluster assignments \( W \) and policies \( \{\pi_1, \pi_2, \dots, \pi_k\} \).
    \end{algorithmic}
\end{algorithm}

The idea of explicitly maintaining $k$ central policies and clustering trajectories based on the likelihood that each central policy generates a target trajectory has also been explored in SORL~\cite{mao2024stylized}. However, it is important to note that SORL employs a continuous assignment matrix $\mathbf{W}$, making it more akin to an EM-style method than to KMeans. As a result, SORL is more prone to subpopulation homogenization, a phenomenon also observed in our experiments (see Table~\ref{tab:NMI}).

Moreover, due to the use of soft assignments via $\mathbf{W}$, SORL requires training each policy network over the entire dataset. In contrast, PG-Kmeans leverages the classification results from the previous round as a prior approximation, and assigns each data point to a separate network, thereby significantly reducing the overall training cost.

\subsubsection{Theoretical Analysis}
\label{subsec: theoretical analysis}

Theoretically, we establish the convergence of PG-Kmeans following a similar argument as K-means \cite{bottou1995convergence}. Before convergence, the loss function strictly decreases after each iteration, ensuring that no identical grouping pattern occurs during the training process. Since the number of possible grouping patterns for \( N \) data points is finite, PG-Kmeans is guaranteed to converge within a finite number of iterations.

\begin{theorem}[Finite-Step Convergence of PG-Kmeans]
\label{thm: finite convergence}
Given a dataset with \( N \) trajectories and \( k \) clusters, the PG-Kmeans algorithm is guaranteed to converge within a finite number of iterations.
\end{theorem}

Rigorous proof deferred to Appendix ~\ref{app: main proof}. We note that in the worst case, the number of iterations required can be as high as \( O\left(K^N\right) \). However, in our experimental setup, the algorithm consistently converged within 20 iterations.

\subsubsection{Implementation details of PG-Kmeans}
In practice, the vanilla version of PG-Kmeans exhibited considerable instability during training. To address this, we adopted a best-of-$N$ selection strategy based on training loss to improve robustness. 

In addition, we employed an overparameterization-and-merging approach to handle scenarios where the true number of clusters $k$ is unknown. Empirically, we found that initializing PG-Kmeans with a slightly overestimated number of hypothetical cluster centers leads to substantial performance improvements (see Section~\ref{subsec:results and discussion}). For further details, please refer to Appendix~\ref{app: Details of PG-Kmeans}.

\subsection{Centroid-Attracted Autoencoder (CAAE)}
\label{subsec: CAAE}





\begin{figure}[ht]
    \centering
    \resizebox{0.9\textwidth}{!}{\input{image/CCAE.tex}}
    \caption{A schematic of our CAAE. The full trajectory $\{(o_i, a_i)\}_{i=1}^H$ is first encoded into a compact latent vector $z$. A learnable Gaussian codebook $\{\mu_i\}$ imposes a regulation loss to keep $z$ near the prior, while the decoder conditions on both $z$ and a single observation $o_i$ to reconstruct the action $\hat a_i$. 
}
    \label{fig: CAAE constructure.}
\end{figure}
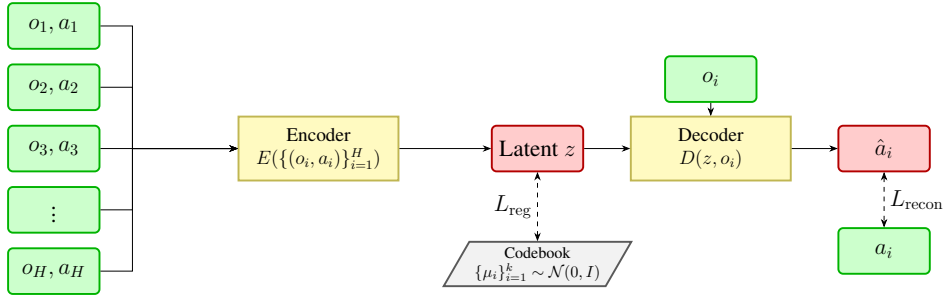
In addition to the explicit trajectory clustering approach employed by PG-Kmeans, we introduce an alternative representation learning-based method, the Centroid-Attracted Autoencoder (CAAE). In CAAE, each input trajectory $\tau$ is first encoded into a latent representation $z = \text{encoder}(\tau)$, which is then combined with the observation sequence and passed through a decoder to reconstruct the trajectory $\tilde{\tau}\sim \text{decoder}(z,o_{1:H})$. 

To regularize the latent space, we impose a constraint motivated by the assumption that each latent variable follows a Gaussian distribution, i.e., $z_i \sim \mathcal{N}(\mu_i, I)$, where $\mu_i$ is selected from a learnable codebook $\{\mu_j\}_{j=1}^k$. This assumption leads to a penalty term that encourages each $z_i$ to be close to one of the centroids in the codebook, thus promoting structured and interpretable embeddings. 

Formally, the CAAE objective is defined as:
\begin{equation}
L(\phi, \theta, \mu) = \sum_{i=1}^N \left( -\sum_{h=1}^H\log \mathbb{P}(a_{i,h} \mid \theta, z_i, s_{i,h}) + \alpha \min_j \|\mu_j - z_i\|^2 \right) - \dfrac{1}{m^2}\sum_{i,j} \min\{1,||\mu_i-\mu_j||_2^2\},
\end{equation}
where $z_i = \text{encoder}(\tau_i, \phi)$, $\theta$ and $\phi$ represents the parameters of encoder and decoder respectively, and the last term is a regularization term. Implementation details can be found in Appendix \ref{app: ae_imp}





\subsection{Comparison Between PG-Kmeans and CAAE}

As one of the first systematic investigations into policy-based trajectory clustering in reinforcement learning, our work introduces and contrasts two complementary approaches—PG-Kmeans and CAAE—each reflecting a distinct design philosophy and offering unique advantages. The differences between these methods can be primarily characterized along two key dimensions:

First, PG-Kmeans assigns each cluster a separate policy network, whereas CAAE employs a shared decoder modulated by latent inputs from a unified encoder. This architectural distinction leads PG-Kmeans to produce sharper, more distinct cluster boundaries, enabling precise policy segmentation. However, this comes at the cost of higher computational complexity and increased sensitivity during training. In contrast, CAAE benefits from parameter sharing, resulting in greater computational efficiency and empirically more stable optimization.

Second, PG-Kmeans performs clustering based on single-step action likelihoods, making it effective at capturing fine-grained distinctions within mixed-intent or concatenated trajectories. CAAE, by comparison, encodes entire trajectories and performs clustering in the latent space of these holistic embeddings. This makes it more adept at modeling global behavioral structure, but less responsive to localized transitions or mode-switching within a single trajectory.

By proposing both PG-Kmeans and CAAE, we offer a dual-perspective framework for policy-based trajectory clustering—one that balances precision and generalization, and accommodates a diverse set of practical scenarios and modeling requirements in offline RL.

%% file: image/CCAE.tex
\begin{tikzpicture}[
  varnode/.style={draw=green!70!black, line width=1.2pt, rounded corners=3pt,
                   minimum width=2.0cm, minimum height=1cm,
                   align=center, fill=green!20, font=\Large},
  rednode/.style={draw=red!70!black, line width=1.2pt, rounded corners=3pt,
                   minimum width=2.0cm, minimum height=1cm,
                   align=center, fill=red!20, font=\Large},
  modelnode/.style={draw=yellow!70!black, line width=1.2pt, rectangle,
                     minimum width=3.5cm, minimum height=1.4cm,
                     align=center, fill=yellow!30, font=\large},
  codebook/.style={draw=gray!70!black, line width=1.2pt, trapezium,
                     trapezium left angle=60, trapezium right angle=120,
                     fill=gray!10, font=\normalsize},
  lossline/.style={draw, dashed, <->},
  >={Stealth}
]

\node[varnode] (seq1) {$o_1, a_1$};
\node[varnode] (seq2) [below=0.3cm of seq1] {$o_2, a_2$};
\node[varnode] (seq3) [below=0.3cm of seq2] {$o_3, a_3$};
\node[varnode] (seqdots) [below=0.3cm of seq3] {$\vdots$};
\node[varnode] (seqH) [below=0.3cm of seqdots] {$o_H, a_H$};

\node[modelnode] (encoder) [right=3cm of seq3] {Encoder\\$E(\{(o_i,a_i)\}_{i=1}^H)$};

\node[rednode] (z) [right=of encoder, xshift=1cm] {Latent $z$};

\node[codebook] (codebook) [below=of z, yshift=-0.5cm] {\shortstack{Codebook \\$\{\mu_i\}_{i=1}^k \sim \mathcal{N}(0,I)$}};

\node[modelnode] (decoder) [right=of z] {Decoder\\$D(z, o_i)$};

\node[rednode] (output) [right=of decoder] {$\hat{a}_i$};

\node[varnode] (trueai) [below=1.2cm of output] {$a_i$};

\draw[->] (seq1.east) -- ++(0.7,0) |- (encoder.west);
\draw[->] (seq2.east) -- ++(0.7,0) |- (encoder.west);
\draw[->] (seq3.east) -- ++(0.7,0) |- (encoder.west);
\draw[->] (seqdots.east) -- ++(0.7,0) |- (encoder.west);
\draw[->] (seqH.east) -- ++(0.7,0) |- (encoder.west);

\draw[->] (encoder) -- (z);

\draw[lossline] (codebook.north) -- (z.south) node[midway,left, font=\Large] {$L_{\mathrm{reg}}$};

\draw[->] (z.east) -- (decoder.west);

\node[varnode, minimum height=1cm] (oi) at ($(decoder.north)+(0,0.8cm)$) {$o_i$};
\draw[->] (oi) -- (decoder);

\draw[->] (decoder.east) -- (output.west);

\draw[lossline] (output.south) -- (trueai.north) node[midway,right, font=\Large] {$L_{\mathrm{recon}}$};

\end{tikzpicture}

%% file: texs/caae.tex

%% file: texs/experiments.tex
\section{Experiments}
\label{sec: experiments}

We evaluate PG-Kmeans and CAAE in both continuous and discrete action spaces and compare it against several traditional clustering methods.

\subsection{Environments and baselines}
\label{subsec: Environments}

\paragraph{Gridworld.}  
We designed three discrete environments and a continous environment with corresponding policies: Takeball, Diagonal and Extra as discrete environment and Pathfollowing as continous environment (Figure ~\ref{fig:enter-label}). In the Takeball environment, the agent selects one of four different balls on the map before navigating to the bottom-right corner as the destination. In the Diagonal environment, the agent is randomly initialized in the top-left corner and must reach the bottom-right corner. In the Extra environment, the agent need to move from a start grid to a terminal grid in a randomly generated map, with two specially marked grids. In Pathfollowing, the agent need to control a ball to be close to a specific point. For each environment, we designed 3-5 different strategies and collected a balanced dataset with them.

To introduce stochasticity, we applied a 0.3 probability of random dynamics at each step in discrete environments, and added Gaussian noise to each move in Pathfollowing. The initial position was sampled uniformly from $[-1.5, -0.5]^2$ in Pathfollowing. The reward function is trivially set to a constant 0. See Appendix \ref{app: dataset_settings} for details.


\paragraph{Gym Environments.}  
For continuous tasks, we also used the widely adopted D4RL dataset \cite{d4rl}. Specifically, we selected the medium-expert datasets from four Gym environments, as they align well with the definition of datasets composed of multiple deterministic policies.
\begin{figure}
    \centering
    \includegraphics[width=0.24\linewidth]{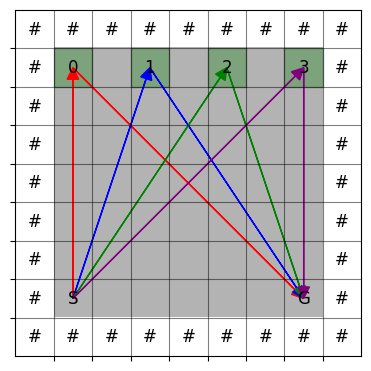}
    \includegraphics[width=0.24\linewidth]{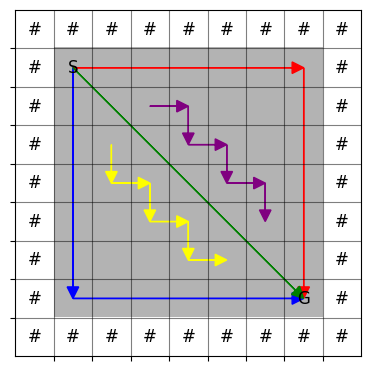}
    \includegraphics[width=0.24\linewidth]{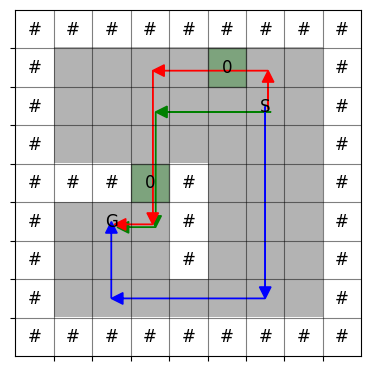}
    \includegraphics[width=0.24\linewidth]{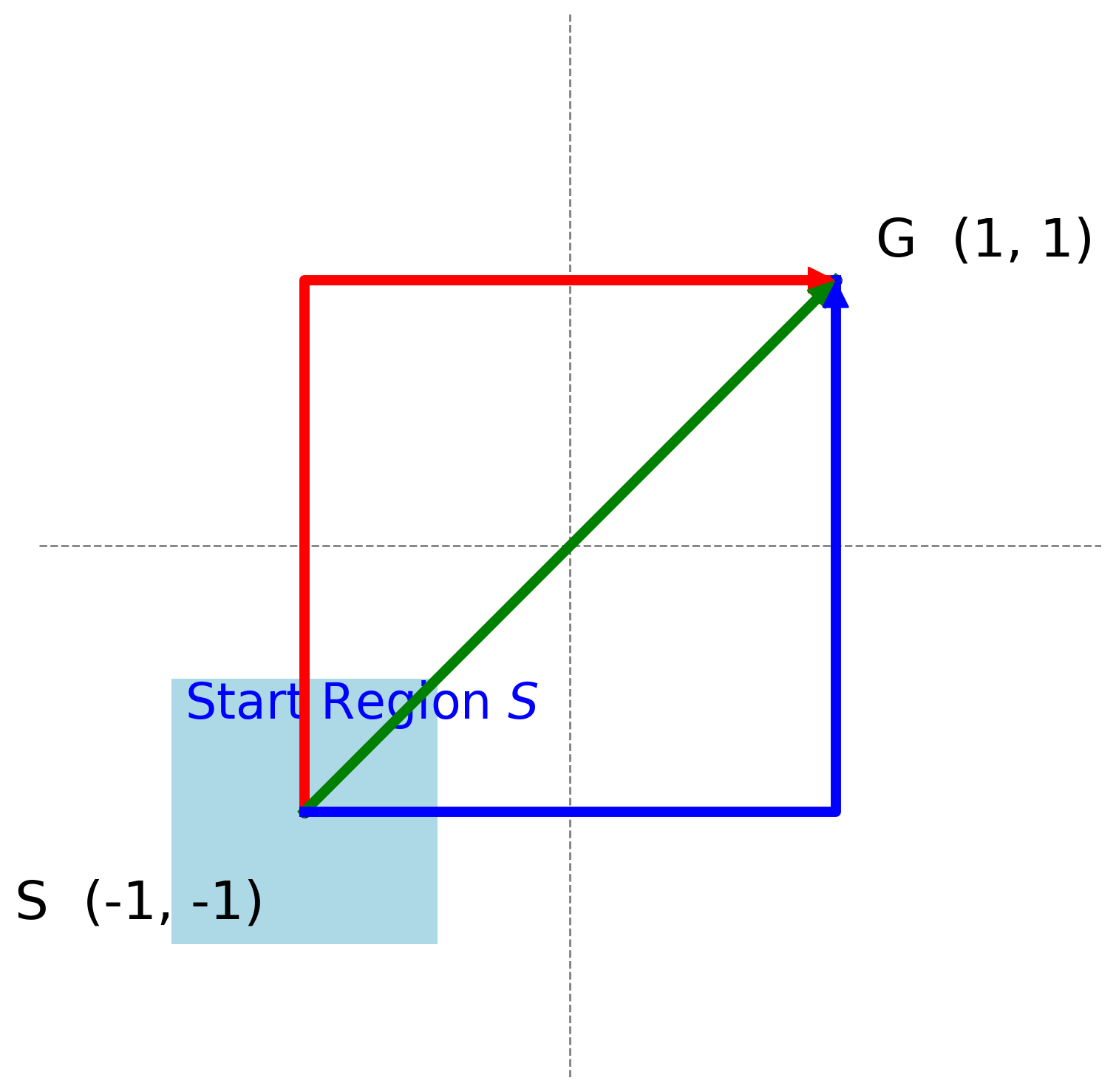}
    \caption{\textbf{Takeball (left 1), Diagonal (left 2), Extra (left 3), Pathfollowing (left 4).} In Takeball, the four ground-truth policies correspond to tendencies to collect each of the four balls, respectively. In Diagonal, Expert 1 (red) prefers moving right first; Expert 2 (blue) prefers moving down first; Expert 3 (green) follows the diagonal; Experts 4 (purple) and 5 (yellow) follow zigzag paths. In Extra, Expert 1 (red) always visits the two special grids; Expert 2 (blue) never visits them; Expert 3 (green) ignores the special grids. In Pathfollowing, three experts try to move corresponding routes first and then follows the route to the terminal point, the special grids are labeled by 0.}
    \label{fig:enter-label}
\end{figure}

\begin{table*}[!t]
    \centering
    \caption{Normalized Mutual Information (NMI) scores for all methods on D4RL and Gridworld environments, The values in the table are presented as mean $\pm$ std. VAE and Return are represent algorithm VAE+Kmeans and Return+Kmeans. Each item are including over at least 10 random seeds. Our algorithms, PG-Kmeans and CAAE, demonstrate consistently high clustering accuracy across the majority of datasets. In contrast, baseline methods such as DEC, SORL, and Return+Kmeans perform well only on a limited subset of datasets.
    }
    \begin{tabular}{|l|c|c|c|c|c|c|c|}
        \hline
        Task        & PG-Kmeans                        &CAAE                      & VAE                       & Return                   & DEC                                        & SORL      \\ \hline
        Halfcheetah & \textbf{0.99} \scriptsize{$\pm$ 0.00}                &\textbf{0.99} \scriptsize{$\pm$ 0.00} & 0.00 \scriptsize{$\pm$ 0.00}   & 0.97 \scriptsize{$\pm$ 0.00}  & 0.95 \scriptsize{$\pm$ 0.01}        &  0.12 \scriptsize{$\pm$ 0.33}         \\ 
        Ant         & 0.92 \scriptsize{$\pm$ 0.00}                &\textbf{0.96} \scriptsize{$\pm$ 0.01} & 0.01 \scriptsize{$\pm$ 0.01}   & 0.05 \scriptsize{$\pm$ 0.00}  & 0.39 \scriptsize{$\pm$ 0.17}        &  0.00 \scriptsize{$\pm$ 0.00}           \\
        Walker2d    & \textbf{0.94} \scriptsize{$\pm$ 0.01}                & 0.88 \scriptsize{$\pm$ 0.10} & 0.00 \scriptsize{$\pm$ 0.00}   & 0.23 \scriptsize{$\pm$ 0.00}  & 0.77 \scriptsize{$\pm$ 0.12}        &  0.08 \scriptsize{$\pm$ 0.24}     \\ 
        Hopper      & \textbf{0.99} \scriptsize{$\pm$ 0.00}                &\textbf{0.99} \scriptsize{$\pm$ 0.01} & 0.00 \scriptsize{$\pm$ 0.00}   & 0.86 \scriptsize{$\pm$ 0.00}  & 0.00 \scriptsize{$\pm$ 0.00}        &  0.84 \scriptsize{$\pm$ 0.26}        \\ \hline
        Diagonal    & \textbf{0.92} \scriptsize{$\pm$ 0.00}                &0.87 \scriptsize{$\pm$ 0.05} & 0.10 \scriptsize{$\pm$ 0.02}   & N/A                             & 0.28 \scriptsize{$\pm$ 0.02}        &  0.18 \scriptsize{$\pm$ 0.05}         \\ 
        Takeball    & \textbf{1.00} \scriptsize{$\pm$ 0.00}                & \textbf{1.00} \scriptsize{$\pm$ 0.00} & 0.00 \scriptsize{$\pm$ 0.00}   & N/A                             & 0.05 \scriptsize{$\pm$ 0.10}        & 0.54 \scriptsize{$\pm$ 0.13}         \\ 
        Pathfollowing &0.12 \scriptsize{$\pm$ 0.11}                   &0.15 \scriptsize{$\pm$ 0.02}  &0.03 \scriptsize{$\pm$ 0.05}    & N/A                             & \textbf{0.26} \scriptsize{$\pm$ 0.10}        & 0.14 \scriptsize{$\pm$ 0.05} \\
        Extra & 0.02 \scriptsize{$\pm$ 0.01}                   &\textbf{0.43} \scriptsize{$\pm$ 0.22}  & 0.04 \scriptsize{$\pm$ 0.05}    & N/A                             &  0.28 \scriptsize{$\pm$ 0.36}        & 0.00 \scriptsize{$\pm$ 0.00} \\
        
        \hline
    \end{tabular}
    \label{tab:NMI}
\end{table*}

\paragraph{Baseline Methods}
To the best of our knowledge, no algorithm has been developed specifically for policy‑based clustering prior to this work. Therefore, in our experiments we compare against a selection of deep clustering methods and functionally analogous approaches. Specifically, we use Deep Embedded Clustering\cite{xie2016unsupervised}, SORL\cite{mao2024stylized}, Return + Kmeans, and VAE + Kmeans as our baselines.

\subsection{Results and Discussion}

\label{subsec:results and discussion}

\begin{figure}
    \centering
    \includegraphics[width=0.495\linewidth]{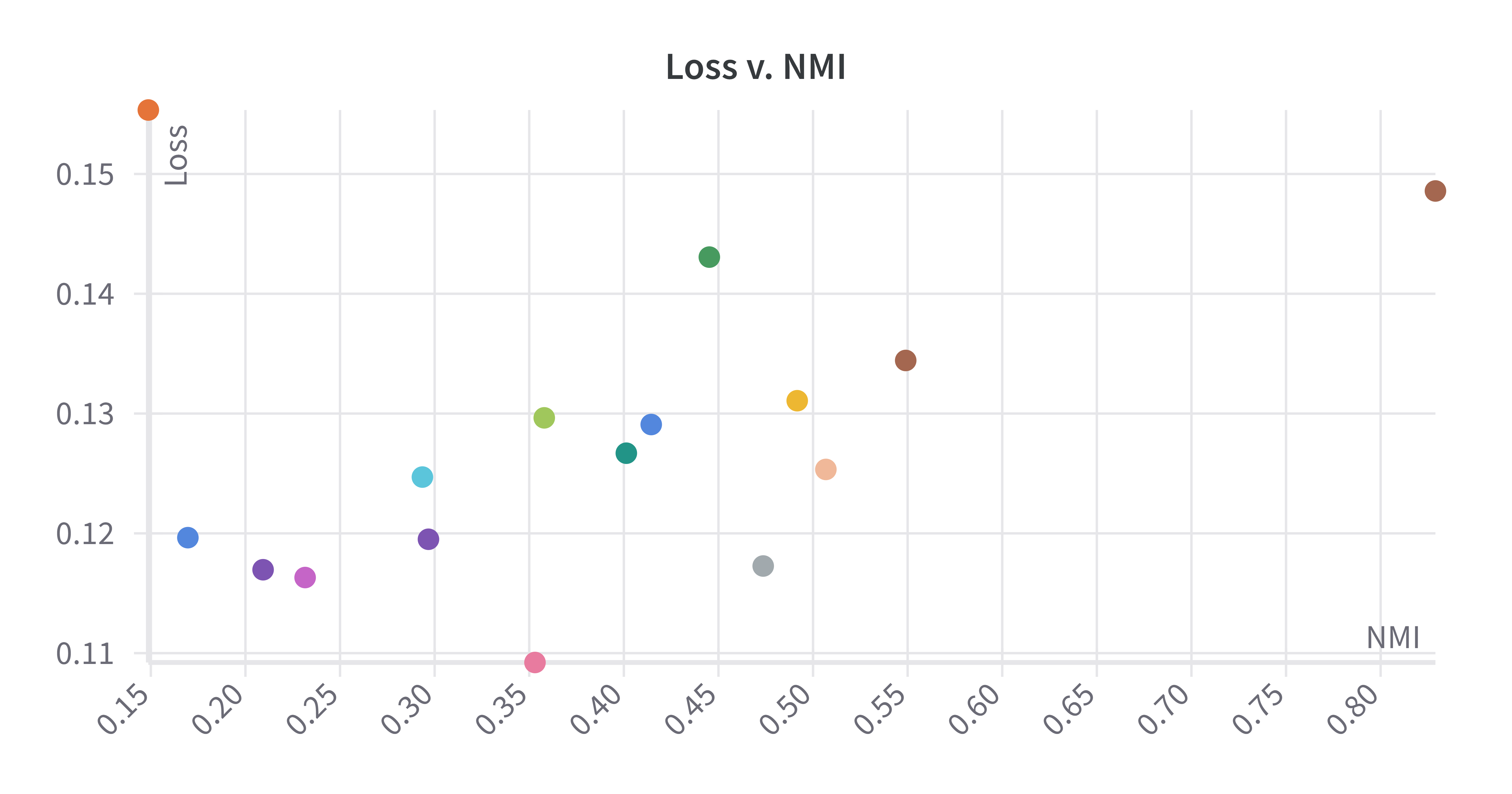}
    \includegraphics[width=0.495\linewidth]{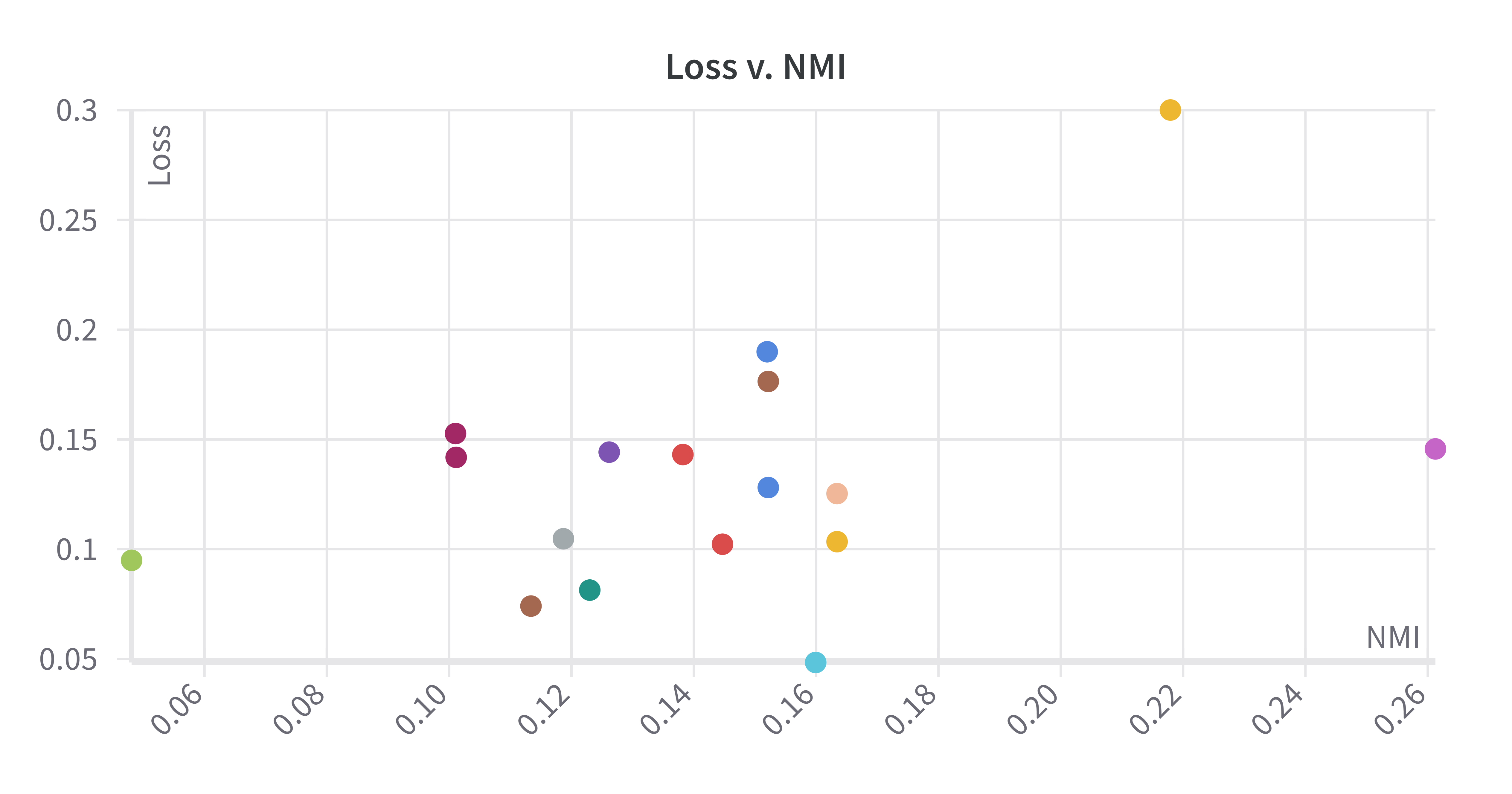}
    \caption{NMI vs. loss in Extra and Pathfollowing with CAAE. The figure shows no significant correlation between NMI and loss, indicating overfitting.}
    \label{fig:caae_extra_loss_v_nmi}
\end{figure}

As shown in Table \ref{tab:NMI}, both PG-Kmeans and CAAE achieve high NMI scores across most of the datasets. Comparison between the action-imitation algorithms PG-Kmeans and SORL reveals that PG-Kmeans achieves better performance than SORL in nearly all environmental configurations. Conventional representation learning-based algorithms, such as DEC and VAE combined with K-means, fail to achieve optimal performance on most datasets. In contrast, CAAE demonstrates significantly superior stability and classification accuracy. 

Meanwhile, we observe that in certain environments, even the best-performing algorithms fail to achieve an NMI close to 1.0. While suboptimal optimization contributes significantly to this gap, the inherent indistinguishability within the dataset also imposes an upper bound below 1.0. For instance, in the Diagonal environment, the behaviors of Expert 3 and Expert 4 are almost identical near the diagonal, making them difficult to separate. Similarly, in Extra, Expert 2 and Expert 3 may follow nearly identical trajectories on certain maps.

Moreover, a major challenge in optimization is overfitting (Figure ~\ref{fig:caae_extra_loss_v_nmi}). When the conflict rate of (s, a) pairs between two policies is low, merging them into a single policy incurs minimal loss increase. For example, in the Pathfollowing environment, different policies only conflict in the start region. This allows a single policy network to model all policies with relatively low loss, even if it does not accurately distinguish them.

In addition to the main results, we conducted an ablation study on topics such as the choice of the cluster number k in both CAAE and PG-Kmeans, as well as the necessity of regularization. Due to space limitations, detailed settings and results are provided in Appendix~\ref{app: ablation study}. The conclusions highlight the effectiveness of regularization and demonstrate that mild overparameterization can lead to performance improvements.

\section{Summary and Limitations}
In this paper, we formalized the problem of policy-based trajectory clustering in offline reinforcement learning and proved that it is NP-complete. We further analyzed how it differs from conventional clustering tasks. To address this challenge, we introduced two methods—PG-Kmeans and CAAE—and evaluated them on carefully curated datasets. Experimental results demonstrate their clear advantages over existing baselines.

A key limitation of this study is the relatively small scale of the experiments. The proposed methods have not yet been tested on large-scale datasets or in complex, real-world environments. Moreover, the algorithms currently lack theoretical convergence guarantees, and the uniqueness of clustering solutions remains an open question. Future work may involve a more rigorous problem formulation, improved strategies to mitigate overfitting, and a comprehensive theoretical analysis of the proposed approaches.

%% file: texs/app_proofs.tex


\section{Appendix: Mathmatical discussion}
\subsection{Proof of Theorem ~\ref{thm: finite convergence}}
\label{app: main proof}
\begin{theorem}[Finite convergence of Algorithm \ref{alg:pg_kmeans}]
Algorithm 1 is guaranteed to converge in $O(k^N)$ iterations.    
\end{theorem}
\begin{proof}
    We denote the original assignment matrix $W$ as $W^0$, policy parameters as $\theta^0$. After iteration $t$, the assignment matrix and parameters become $W^t$ and $\theta^t$. Assume that the policy converges at after $T\in \mathbb{Z}^+\cup\{\infty\}$ iterations.
    From the definition we see that $\forall t=1,2,\cdots T$,
    \begin{align}
        \max_{\theta} J(W^{t-1}, \theta)&= J(W^{t-1}, \theta^t)\\
        &< \max_{W}J(W, \theta^t)\\
        &=J(W^t, \theta^t)\\
        &\leq \max_{\theta} J(W^t, \theta).
    \end{align}
    Therefore $\{\max_{\theta} J(W^{t-1}, \theta)\}_{t=1}^T $ is a strictly increasing sequence. Because there are at most $k^N$ different valid values for $W$, the sequence length $T$ is no larger than $k^N$, which concludes the proof.
\end{proof}

\subsection{Policy-based trajectory cluster and k-coloring}
\label{app: k-coloring}
One fundamental reason why K-means clustering cannot be directly applied to policy-based trajectory clustering is the absence of a well-defined distance metric in trajectory space. A natural idea is to define the "distance" between two trajectories based on their compatibility, i.e., whether they could have been generated by the same policy:

\begin{equation}
\label{equ:distance}
d(\tau_1, \tau_2) = 
\begin{cases} 
1, & \exists(s, a) \in \tau_1, (s, a') \in \tau_2, a \neq a' \\ 
0, & \text{otherwise} 
\end{cases}
\end{equation}

However, this definition does not satisfy the triangle inequality, i.e., 

\begin{equation}
|d(x, y)| \leq |d(x, z)| + |d(z, y)|
\end{equation}

which means it is not a proper metric and provides limited mathematical utility.

For example, consider the trajectories $x = [(s_1, a_1)]$, $y = [(s_1, a_2)]$, and $z = [(s_2, a_1)]$. In this case, we observe that:

\begin{equation}
|d(x, y)| = 1 \not\leq |d(x, z)| + |d(z, y)| = 0.
\end{equation}

We also note that, in the absence of additional assumptions, Equation (\ref{equ:distance}) encapsulates all the reliable information available in the dataset. That is, any clustering scheme $W$ that satisfies the condition that the total intra-cluster distance is zero provides a valid solution:

\begin{equation}
D(W) = \sum_{k=1}^K\sum_{i=1}^N\sum_{j=1}^N w_{i, k}w_{j, k}d(\tau_i, \tau_j) = 0.
\end{equation}

This formulation precisely aligns with the definition of the K-coloring problem.

Furthermore, we can construct a simple proof for the following theorem:

\begin{theorem}[Reduction from K-coloring to Policy Clustering]
For any K-coloring problem with $N$ nodes and a maximum degree of $d$, there exists a Markov Decision Process (MDP) with $|\mathcal{S}|\geq 2d$ and $|H|\geq d$, along with a corresponding dataset, such that the dataset can be generated by $K$ distinct policies, and its valid clustering corresponds to a solution of the original K-coloring problem.
\end{theorem}
\begin{proof}
    We can construct a dataset by following steps:
    \begin{algorithm}[ht!]
    \caption{Reduction from K-coloring to policy clustering}
    \label{alg:reduction}
    \begin{algorithmic}[1]
        \STATE \textbf{Input:} a graph $G=(V, E), s.t. \max(\mathrm{degree}(v))=d$
        \STATE \textbf{Initialize} Initialize a list $t_i$ for each node $v_i$.
        \FOR{$i = 1$ to $|V|$}
            \FOR{$j= i+1$ to $|V|$}
                \STATE If $(v_i, v_j)\in E$, find the smallest integer $l$ such that $s_l\notin t_i\cup t_j$
                \STATE Append $t_i$ with $(s_l, a_1)$, $t_j$ with $(s_l, a_2)$
            \ENDFOR
        \ENDFOR
        \STATE Pad $\{t_i\}_i$ to length $H$ and concatenate them to get trajectories $\{\tau_i\}_i$.
        \STATE \textbf{Output:} Dataset for clustering $\{\tau_i\}_{i=1}^{|V|}$.
    \end{algorithmic}
    \end{algorithm} 
    
    Because $|t_i\cup t_j|\leq \mathrm{degree}(v_i)+\mathrm{degree}(v_j)\leq 2d$, $l$ would never take value over $2d$. And $\max_i(|t_i|)\leq \max(\mathrm{degree}(v))=d$, so horizon $H>d$ is enough.
\end{proof}

This result implies that the general policy-guided clustering problem is NP-complete, for we have known that 3-coloring with $d\geq 2$ is NP-complete.

\subsection{Policy ambiguity}
\label{app: policy ambiguity}

The above reduction from k-coloring to policy-based clustering demonstrates that policy-based clustering may have multiple solutions. However, this is not necessarily unacceptable. As long as the center policies of the clusters remain consistent, different clustering solutions merely arise due to the fact that these policies exhibit identical expressions in certain instances. The more critical issue is that not only can the trajectory clustering solutions be non-unique, but the center policies themselves may also constitute entirely different constructure. Consider the following example:

Let us examine a contextual bandit setting with only two states and two actions. This environment allows for exactly four distinct deterministic policies and four possible trajectories:

\begin{itemize}
    \item \textbf{Policy 1:} $(s_1, a_1), (s_2, a_1)$
    \item \textbf{Policy 2:} $(s_1, a_1), (s_2, a_2)$
    \item \textbf{Policy 3:} $(s_1, a_2), (s_2, a_1)$
    \item \textbf{Policy 4:} $(s_1, a_2), (s_2, a_2)$
\end{itemize}

Using either Policy 1 and Policy 4, or Policy 2 and Policy 3, we can generate all four possible trajectories. Thus, for any dataset collected from this contextual bandit, there will always exist at least two valid clustering solutions. Furthermore, when the occurrence probabilities of $s_1$ and $s_2$ are equal, these two clustering solutions are completely irrelevant with each other.

The experimental results demonstrate that the non-uniqueness of policy combinations is not merely a theoretical possibility but frequently manifests in practical datasets. In the deprecated environment analogous to Takeball, we observe a modified scenario where four balls are randomly permuted across four fixed positions, while the expert policy $\pi_i$ still maintains the strategy of collecting the $i$-th numbered ball and returning. 

Notably, the neural network rapidly learns a distinct policy categorization approach: instead of differentiating policies by ball numbering, it develops four position-specific policies that collect balls based on their spatial locations. Although both policy combinations generate identical trajectory distributions within the environment, the position-based strategy exhibits significantly simpler implementation requirements. Each agent in this paradigm only needs to learn navigation to a fixed location, avoiding the more complex task of simultaneously identifying specific ball numbers and selecting among four potential destinations.

Through empirical analysis, we find that random initialization almost invariably leads the network to discover the position-based discrimination method. In contrast, achieving number-based trajectory differentiation requires initializations remarkably close to the target clustering configuration (specifically, Normalized Mutual Index (NMI) > 0.6 in our experiments). 

This fundamental non-uniqueness in policy combinations constitutes a critical challenge that directly impacts problem solvability, as it introduces substantial ambiguity in policy identification and may lead to suboptimal solutions that fail to capture the intended behavioral semantics.

%% file: texs/app_experiments.tex
\section{Appendix: Implementation details}
\label{app: Implementation details}
\subsection{Dataset settings}
\label{app: dataset_settings}
\paragraph{D4RL}
We directly use the medium-expert datasets provided by D4RL for four Gym environments. Each dataset consists of 1,000,000 timesteps. The original datasets are not segmented into episodes nor labeled with their generating policies. Therefore, we divide the dataset into episodes based on the `Done` signal and assign labels accordingly.

Since the first half of the dataset corresponds to medium-level data and the latter half to expert-level data, we partition the dataset by maximizing the average return of the first and second halves. However, as it is possible that some of the initial expert episodes achieve relatively low returns, this method may introduce an error of up to five episodes (approximately 5,000 timesteps). Nonetheless, this minor misclassification has a negligible impact on the final NMI and does not affect the quantitative conclusions presented earlier.

\paragraph{Gridworld}
The Diagonal, Takeball and Extra environments are implemented in JAX. To align with the MDP framework, we use a fixed 9×9 grid map. The action space consists of five discrete actions: movement in four directions and a standby action. The maximum episode length is set to 40 timesteps, and an episode terminates immediately when the agent reaches the goal position. Additionally, to simulate stochastic dynamics, each step has a 0.3 probability of ignoring the action input and taking a random action.

The observation space contains the full state information. In Diagonal, the observation is represented as a 9×9×3 matrix obtained by stacking the wall map, agent map, and goal map. In Takeball, the observation further includes four additional one-hot matrices to encode the positions of four balls, resulting in a 9×9×4 representation. In Extra, the observation further includes two additional one-hot matrices to encode the positions of two special grid, resulting in a 9×9×2 representation.

For Diagonal and Takeball environments, all nine expert policies are rule-based, relying solely on the current state without considering historical information. They are implemented in JAX. For each expert, we collect 20,000 trajectories for evaluation. As a result, the Diagonal dataset consists of a total of 100,000 trajectories, while the Takeball dataset contains 80,000 trajectories. And for Extra environment. We train three policies using PPO algorithm by setting positive/zero/negative rewards for the agent to reach the special grid. Other details are the same as Diagonal and Takeball. The Extra dataset consists of a total of 60,000 trajectories.

\subsection{Dataset Collection Methods in Gridworld}

\paragraph{Diagonal} There are five different rule-based policies:

\begin{enumerate}
    \item Always move to the right, until the wall is reached.
    \item Always move to the down, until the wall is reached.
    \item Move to the right iff at the left-down half of the gridworld.
    \item Move to the right iff at the black grid, if we seem the gridworld as a chess board.
    \item Move to the right iff at the white grid.
\end{enumerate}

\paragraph{Takeball} There are four different rule-based policies, i-th policy will pick the i-th ball first.

\paragraph{Extra} We use PPO algorithm to train three policies by giving positive/zero/negative rewards for the agent to reach the special grid. The rewards are set to +10, 0 or -10, respectively. When the agent reach the goal, the reward is set to +10. The agent will get a penalty of -0.3 for each step and each unit of distance to the goal.

\paragraph{PathFollowing} We use PPO algorithm to train three policies by giving extra penalty for the agent when it off path. For each step, the reward is negative distance to the goal add the following penalty for each policy:

\begin{enumerate}
    \item No extra penalty.
    \item 5 times of distance to the polyline $(-1,-1),(-1,1),(1,1)$.
    \item 5 times of distance to the polyline $(-1,-1),(1,-1),(1,1)$.
\end{enumerate}

These policies are corresponding to: No preference, prefer to go up first and prefer to go right first.

\subsection{Baseline Methods}
\label{app: baseline methods}

\paragraph{Deep Embedded Clustering (DEC).}  
Deep Embedded Clustering (DEC) \cite{xie2016unsupervised} is a widely used deep clustering method that integrates representation learning with clustering optimization. It first pretrains an autoencoder to map high-dimensional data into a low-dimensional latent space. Clustering is then performed in this learned space by iteratively refining cluster assignments using a self-training objective. Specifically, DEC employs a Student's t-distribution to measure the similarity between data points and cluster centers and minimizes a Kullback-Leibler (KL) divergence loss to refine embeddings for more compact and well-separated clusters.

In this work, we adapt DEC for policy clustering by modifying its encoding process to focus on capturing policy-related information rather than trajectory-level dynamics. The details are attached in Appendix \ref{app: ae_imp}.

\paragraph{SORL} Stylized Offline Reinforcement Learning(SORL) \cite{mao2024stylized} is a offline reinforcement learning algorithm that designed to learn diverse policies with varied styles. The first step of the algorithm is soft clustering, that is assigning weight to each trajectories on clusters. In this experiment, we impliment the clustering step only, and use hard clustering instead of soft clustering to represent SORL algorithm.

\subsection{Details of PG-Kmeans Implementation}
\label{app: Details of PG-Kmeans}
\subsubsection{Best-of-N PG-Kmeans}
\label{subsec:best of N pgkmeans}

Similar to K-means, PG-Kmeans is highly sensitive to initialization. Moreover, since PG-Kmeans optimizes a neural network for each cluster, small clusters are particularly prone to severe overfitting and non-convex optimization issues, which can ultimately lead to their disappearance.

This instability makes PG-Kmeans less robust during training. To mitigate this issue, we propose the Best-of-N technique. As shown in Algorithm~\ref{alg:best of n pg_kmeans}, we use the final objective function \( J(W,\theta) \) as an internal metric to evaluate the quality of different runs in the absence of ground truth. The best clustering result is then selected for output.

\begin{algorithm}[ht!]
    \caption{Best-of-N PG-Kmeans}
    \label{alg:best of n pg_kmeans}
    \begin{algorithmic}[1]
        \STATE \textbf{Input:} Input for PG-Kmeans and number of runs \( N \).
        \FOR{$t = 1$ to $N$}
            \STATE Run PG-Kmeans (Algorithm~\ref{alg:pg_kmeans}) and obtain the corresponding output and \( J(W, \theta) \).
        \ENDFOR 
        \STATE \textbf{Output:} Result from the run with the highest \( J(W, \theta) \).
    \end{algorithmic}
\end{algorithm}

\subsubsection{Over-parameterization and Merging}
\label{subsec: over-parameterization and merging}

To further improve optimization performance, we introduce the over-parameterization and merging technique. The algorithm faces two primary challenges: (1) center policies often overfit to short and low-density trajectories, causing them to become trapped in incorrect clusters, and (2) with suboptimal initialization, different clusters may be mistakenly merged during clustering. Over-parameterization mitigates these issues by initializing the cluster count \( k \) larger than the true number of clusters \( k^* \), enhancing clustering robustness. However, this approach introduces a new challenge: data points from the same cluster may become dispersed, reducing clustering quality. The merging step addresses this issue by consolidating similar clusters, counteracting the adverse effects of a large \( k \) while preserving clustering coherence.  

For detailed experimental results, see Section~\ref{subsec:results and discussion}.

\begin{algorithm}[ht!]
    \caption{Merge Clusters}
    \label{alg:merge}
    \begin{algorithmic}[1]
        \STATE \textbf{Input:} Datasets and center policies $\{(\mathcal{D}_i, \pi_i)\}_{i=1}^k$, target number of clusters $k^*$.
        \FOR{$i= 1$ to $k-k^*$}
            \STATE Find indices $i, j = \arg\min_{i\neq j}\sum_{\tau \in \mathcal{D}_j}\log \mathbb{P}(\tau\mid \theta_i)$.
            \STATE Merge dataset $j$ into dataset $i$, discard policy $\pi_j$. Renumber datasets and policies from 1 to $k-i$.
        \ENDFOR
        \STATE \textbf{Output:} Datasets and center policies $\{(\mathcal{D}_i, \pi_i)\}_{i=1}^{k^*}$.
    \end{algorithmic}
\end{algorithm}

\subsection{Details of Representation Learning-based Algorithm Implementation}
\label{app: ae_imp}
In practice, we found that the most important information to distinguish different strategies for a trajectory are given by a contiunous sub-trajectory. Inspired by this observation, we design the encoder model as follows: use a GRU network to encode every prefix and use attention mechanism to get the weighted sum of the outputs of GRU. That is, if the output of GRU is $Y=(y_1,\dots, y_H)$, we will train three matrix $Q,K,V$, then the output of attention is:

$$
a=\operatorname{softmax}(Q (KY)^\top) (VY)
$$

During decoding, we condition on the observation at each timestep along with the embedding to generate the action distribution, using the negative log-likelihood of the true action under this distribution as the reconstruction loss. Under this formulation, the encoder is explicitly designed to capture only policy information, i.e., the action generation mechanism.

To highlight the advantages of CAAE over other representation learning-based algorithms such as DEC and VAE, we employed the same Encoder and Decoder structures used in CAAE for DEC and VAE.

\subsection{Network Architectures and Training Details}
\label{app: network architectures and training details}
All networks are trained using the Adam optimizer with a learning rate of \( 1 \times 10^{-3} \). The parameter $\alpha$ of CAAE is set to 1. For the D4RL environments, observations are pre-normalized using statistics from the training set, while no normalization is applied to GridWorld inputs.

\paragraph{Policy Networks.}  
For continuous environments, policies use a \texttt{MultivariateNormalDiag} distribution, where two fully connected (FC) layers process extracted features to produce \texttt{action\_logits} and \texttt{action\_std}. For discrete environments, policies use a \texttt{Categorical} distribution, with a single FC layer mapping extracted features to \texttt{action\_logits}.

\paragraph{Feature Extractors.}  
The feature extractors for different models are summarized in Table~\ref{tab:network_architecture}. All networks use the ReLU activation function.

\begin{table}[h]
    \centering
    \caption{Network architectures for different models.}
    \label{tab:network_architecture}
    \begin{tabular}{|l|c|c|c|}
        \hline
        Model       & Fully Connected Layers & GRU Hidden Size & GRU Output Size \\ \hline
        PG-Kmeans  & (128, 128) & 64  & 128  \\ 
        VAE/DEC/CAAE-Encoder & (128, 128)  & 64  & 128   \\ 
        VAE/DEC/CAAE-Decoder & (128, 32, 32)  & N/A & N/A \\ 
        \hline
        \hline
        \multicolumn{2}{|l|}{VAE/DEC/CAAE-Encoder-Attention-Heads}&\multicolumn{2}{|c|}{2}\\
        \multicolumn{2}{|l|}{VAE/DEC/CAAE-Encoder-Attention-Feature-Size}&\multicolumn{2}{|c|}{2×8}\\
        \hline
    \end{tabular}
\end{table}

\section{Appendix: Full Results and Ablation Studies}
\subsection{Full Results}
\label{app: results}
Here, we provide experimental results for single-run PG-Kmeans and Best-of-5 DEC for comparison with PG-Kmeans. Notably, even when DEC is allowed to enhance its performance by repeatedly running and selecting the best result, it still fails to achieve satisfactory classification in certain environments. This limitation arises because DEC lacks an explicit policy clustering representation, meaning that its embedding-based clustering approach does not guarantee successful categorization in complex environments. Also, as we discussed in Section~\ref{subsec:results and discussion}, there is no significant negative correlation between loss and NMI for CAAE algorithm, so we didn't do Best-of-5 experiment for CAAE algorithm.

We also evaluated the performance of PG-Kmeans' centroid strategy, the experimental results are presented in Table~\ref{table:pgkmeans_k_value_results}, with classification histograms shown in Figure~\ref{fig:histogram}. Across all four environments, PG-Kmeans successfully classified trajectories generated by different policies. When $k > 5$, the accuracy of classification exceeded 50\% in all cases. In most cases, the best output policy reaches the performance of agents trained with 10\% BC\citep{tarasov2022corl}.
\begin{table*}[!t]
    \centering
    \footnotesize
    \begin{tabular}{|l|c|c|c|c|}
        \hline
        Task        & Single-run PG-Kmeans & PG-Kmeans   & DEC & DEC (best of 5)* \\ \hline
        halfcheetah & 0.495 (\scriptsize{$50\% \in [0.99, 1.00]$}) & \textbf{0.989} \scriptsize{$\pm$ 0.000}  & 0.945 \scriptsize{$\pm$ 0.031}  & \textbf{0.989} \scriptsize{$\pm$ 0.000} \\ 
        ant         & 0.745 (\scriptsize{$85\% \in [0.83, 0.94]$})  & \textbf{0.924} \scriptsize{$\pm$ 0.003}  & 0.390 \scriptsize{$\pm$ 0.512}  & 0.756 \scriptsize{$\pm$ 0.003} \\ 
        walker2d    & 0.557 (\scriptsize{$50\% \in [0.70, 0.99]$})  & \textbf{0.942} \scriptsize{$\pm$ 0.005}  & 0.767 (\scriptsize{$70\% \in [0.99, 1.00]$})  & \textbf{0.990} \scriptsize{$\pm$ 0.000} \\ 
        hopper      & 0.258 (\scriptsize{$25\% \in [0.99, 1.00]$})  & \textbf{0.994} (\scriptsize{$80\% \in [0.99, 1.00]$})  & 0.000 \scriptsize{$\pm$ 0.000}  & 0.000 \scriptsize{$\pm$ 0.000} \\ \hline
        Diagonal    & 0.892 \scriptsize{$\pm$ 0.032}  & \textbf{0.920} \scriptsize{$\pm$ 0.001}  & 0.276 \scriptsize{$\pm$ 0.017}  & 0.287 \scriptsize{$\pm$ 0.001} \\ 
        Takeball    & \textbf{0.996} \scriptsize{$\pm$ 0.002}  & \textbf{0.997} \scriptsize{$\pm$ 0.000}  & 0.054 \scriptsize{$\pm$ 0.097}  & 0.175 \scriptsize{$\pm$ 0.027} \\ 
        \hline
    \end{tabular}
    \caption{Normalized Mutual Information (NMI) scores for all methods on D4RL and GridWorld environments, averaged over at least 10 random seeds. For experiments with clearly distinguishable success or failure outcomes, we report the probability of success along with the NMI range conditioned on successful trials. For all other cases, we assume Gaussian noise and report the 95$\%$ confidence interval. *Note: DEC (best of 5) is not a practically feasible algorithm, as ground truth labels are unavailable in real-world clustering scenarios. Without access to true labels, selecting the best result is infeasible. In contrast, PG-Kmeans (best of N) remains a valid approach, as it relies on the final objective function \( J(W, \theta) \) as an internal metric to determine the best clustering result for output.}
    \label{tab:app-NMI}
\end{table*}

\begin{table*}[!t]
    \centering
    \begin{tabular}{|l|c|c|c|c|c|c|}
        \hline
        Task        & Single-run PG-Kmeans & PG-Kmeans  & BC          & 10$\%$-BC     & AWAC     & CQL      \\ \hline
        halfcheetah & 56.5      & 83.4                  & 55.9        & 90.1          & 93.6     & 95.6     \\ 
        ant         & 76.3      & 130.4                 & /           & /             & /        & /        \\
        walker2d    & 44.3      & 104.5                 & 99.0        & 108.7         & 49.4     & 109.6    \\ 
        hopper      & 42.48     & 87.2                  & 52.3        & 111.2         & 52.7     & 99.3     \\ 
        \hline
    \end{tabular}
    \caption{Normalized returns of PG-Kmeans and baseline offline RL algorithms on D4RL datasets. Results for other methods are taken from the CORL benchmark \cite{tarasov2022corl}. PG-Kmeans is initialized with four cluster centers, and the reported returns are obtained by evaluating the center policies of the two non-trivial (non-zero) clusters in each environment. The results indicate that PG-Kmeans significantly outperforms Behavior Cloning. However, for efficiency reasons, center policies are not fully trained to convergence during optimization. As a result, even with nearly perfect clustering, the learned center policies may not always serve as optimal action generators. Notably, PG-Kmeans operates as a semi-supervised learning method, requiring only a minimal amount of return signals for evaluation, yet achieving performance comparable to fully supervised RL algorithms.}
    \label{app:tab Normalized returns}
\end{table*}

\begin{table*}[!t]
    \centering
    \caption{NMI of CAAE with different cluster counts \( k \) in different environments. The results are averaged over 10 random seeds.}
\begin{tabular}{|l|c|c|c|c|c|}
\hline
Task &  k=4 &  k=5 &  k=6 &  k=7 &  k=8 \\ \hline
Ant         & \textbf{0.88} & 0.87 & 0.83 & 0.81 & 0.79 \\
HalfCheetah & \textbf{0.90} & 0.85 & 0.81 & 0.80 & 0.75 \\
Hopper      & \textbf{0.66} & 0.55 & 0.50 & 0.47 & 0.46 \\
Walker2d    & \textbf{0.79} & 0.74 & 0.71 & 0.69 & 0.64 \\
Diagonal         &  N/A  & 0.82 & \textbf{0.87} & 0.80 & 0.83 \\
Takeball                  & \textbf{1.00} & 0.95 & 0.90 & 0.88 & 0.85 \\ 
PathFollowing  & 0.18 & 0.22 & 0.27 & 0.34 & \textbf{0.35} \\ 
Extra  & 0.38 & 0.34 & \textbf{0.43} & 0.36 & 0.37 \\
\hline
\end{tabular}
    \label{tab:caae_k_value}
\end{table*}

\begin{table*}[!t]
    \centering
    \caption{Running time of PG-Kmeans and CAAE. The unit of time is minute in this table. The GPU used is NVIDIA RTX A6000, 48GB.}
    \begin{tabular}{|l|c|c|c|c|c|c|c|c|}
        \hline
        Algorithm & HalfCheetah & Ant & Walker2d & Hopper & Diagonal & Takeball & PathFollowing & Extra \\ \hline
        PG-Kmean&17&24&32&70&22&5.7&24&9.7\\
        CAAE&6.6&10&9.7&15&2.4&1.9&38&2.3\\
        \hline
    \end{tabular}
    \label{tab:running_time}
\end{table*}

\subsection{Ablation Studies}
\label{app: ablation study}
\subsubsection{Regularization in CAAE}

To mitigate the overfitting issue, we introduced a regularization term to the codebook in CAAE. Without this regularization, for the codebook $\mu$, the last layer of the encoder $C$ and the first layer of the decoder $D$, and for a real number $0<\lambda<1$. We can find the recon loss will not change when $C'=\lambda C,D'=\lambda D$, but the codebook loss will being smaller if $\mu'=\lambda \mu$. This phenomenon will lead the codebook to collapse to a single point, which is not desirable. To prevent this, we add a regularization term to the codebook loss.

We can see from Table \ref{tab:regularization_nmi}, regularized CAAE is more robust and achieves better performance in some hard environments. And regularization will not hurt the performance in almost all the environments.

\begin{table*}[!t]
    \centering
    \caption{NMI of CAAE with or with out codebook regularization.}
    \begin{tabular}{|l|c|c|c|c|c|c|c|c|}
        \hline
        Regularization & HalfCheetah & Ant & Walker2d & Hopper & Diagonal & Takeball & PathFollowing & Extra \\ \hline
        W/o &0.99&0.96&0.65&0.99&0.84&1.00&0.14&0.39\\
        W/ &0.99&0.96&0.88&0.99&0.86&1.00&0.15&0.43\\
        \hline
    \end{tabular}
    \label{tab:regularization_nmi}
\end{table*}

\subsubsection{Impact of Initial Cluster Count \( k \).}  
\label{app: k_value}
We further examined the impact of the initial cluster count \( k \) on the clustering performance to assess the algorithm's robustness when the true number of categories \( k^* \) is unknown or inaccurately estimated. Due to poor performance in the GridWorld environment, this analysis focuses solely on the Gym dataset.

As shown in Figure~\ref{fig: ablation study}, DEC is highly sensitive to \( k \). On the Gym dataset, increasing \( k \) to 10 results in a significant performance drop (0.8 \( \to \) 0.5). This degradation primarily stems from the emergence of multiple active cluster centers, which fragment a single true class—an inherent limitation of K-means-based methods. PG-Kmeans exhibits a similar issue but to a lesser extent. On the Gym dataset, it does not naturally disperse, and in the more challenging GridWorld dataset, its NMI decreases by less than 0.1 even when \( k \) is increased to three times the ground-truth value. Furthermore, after applying the merging process, PG-Kmeans experiences almost no performance degradation. In fact, when \( k \) is slightly larger than \( k^* \), it benefits from reduced overfitting, leading to improved performance.

As illustrated in Table \ref{tab:caae_k_value}, CAAE's clustering performance does not exhibit consistent patterns like PG-Kmeans or DEC, but rather demonstrates distinct and divergent trends across different datasets – showing a clear positive correlation in Pathfollowing environments, a negative correlation in Walker2d, and only minor fluctuations likely caused by variance in Extra. Generally, more challenging environments benefit more from increasing the k-value, while simpler ones show the opposite trend. We believe that this phenomenon may stem from the aforementioned limitations of the K-means method, which not only compromises model performance but also mitigates the overfitting issue discussed earlier.

Another advantage of increasing \( k \) is the enhanced stability of PG-Kmeans, as it reduces the likelihood of mode collapse. Across all tested Gym environments, we observe that the probability of correctly identifying the two policies increases as \( k \) grows.

Empirically, for PG-Kmeans algorithm, the optimal choice of \( k \) should be slightly larger than \( k^* \), as this strikes a balance between improving stability, reducing overfitting, and minimizing classification accuracy loss.



\begin{figure*}
    \centering
    \includegraphics[width=0.24\linewidth]{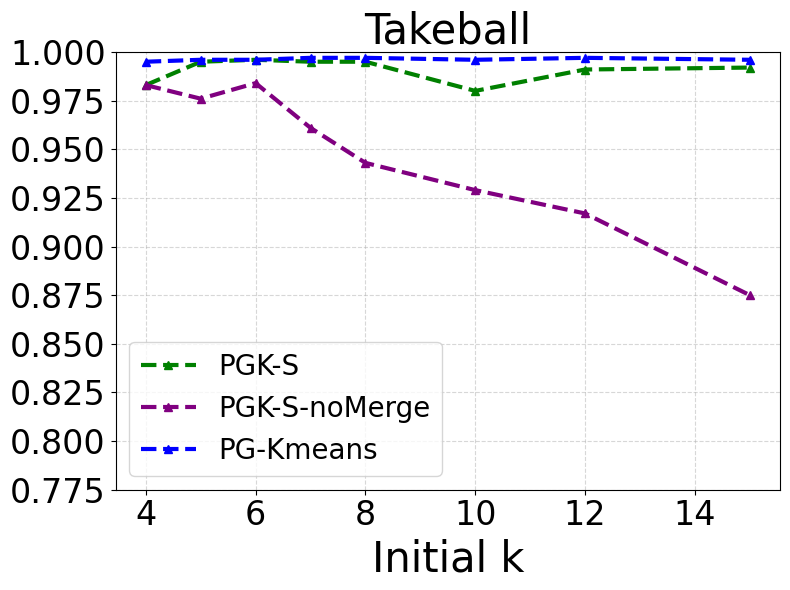}
    \includegraphics[width=0.24\linewidth]{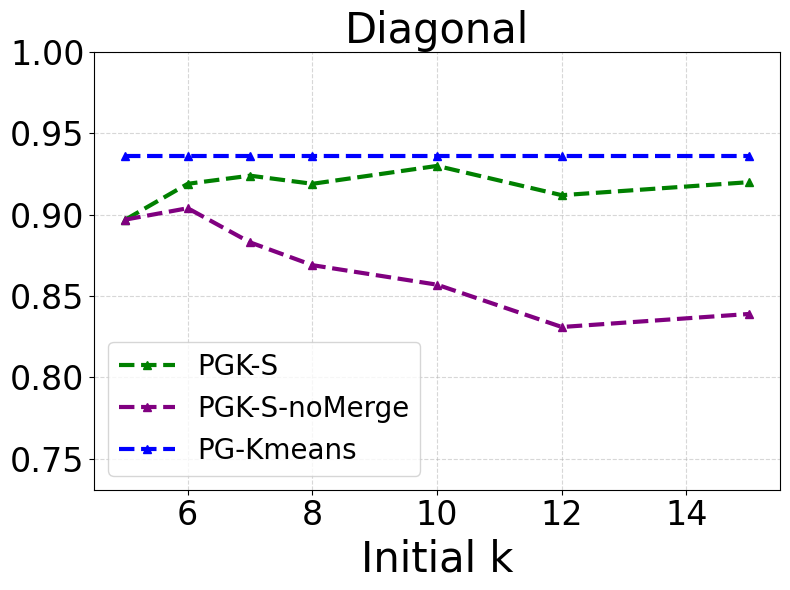}
    \includegraphics[width=0.24\linewidth]{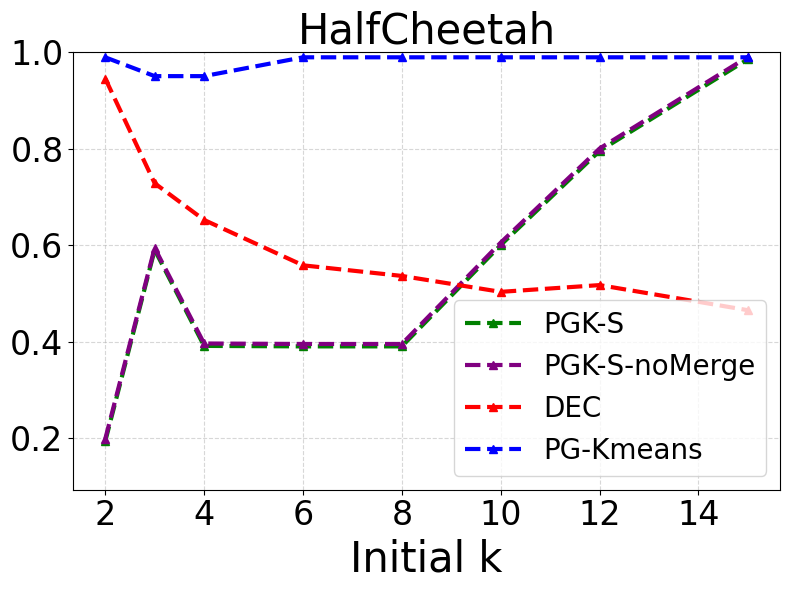}
    \includegraphics[width=0.24\linewidth]{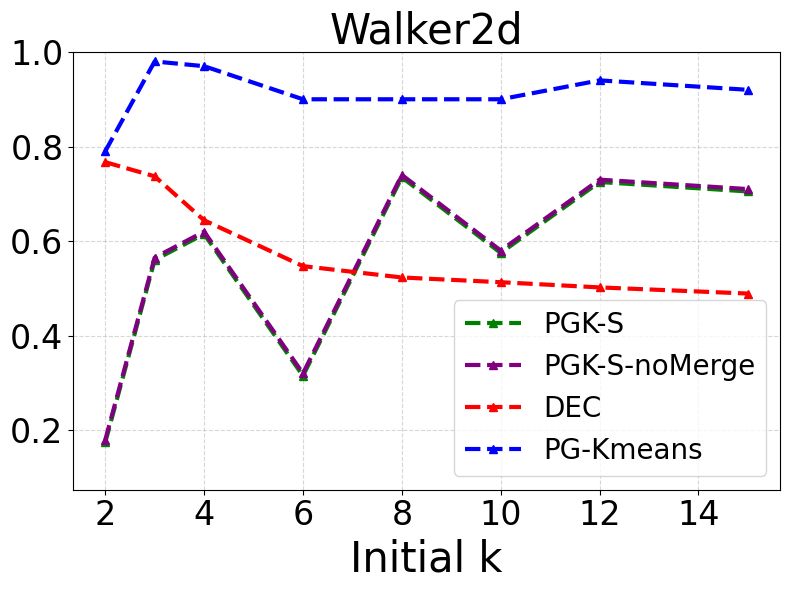}
    \caption{
    The impact of different initial cluster counts \( k \) on the final NMI. All reported values are averaged over 10 random seeds. PGK-S represents the results of single-run PG-Kmeans, while PGK-S-noMerge represents single-run PG-Kmeans without merging. As shown in the figure, increasing the initial cluster count \( k \) leads to a noticeable decline in performance metrics for both PG-Kmeans and DEC, with the effect being more pronounced in DEC. This is because a single class is more likely to be split into multiple subclasses that cannot be automatically merged. However, after introducing the merging process, this issue no longer significantly affects PG-Kmeans' performance. Additionally, a higher initial \( k \) reduces the probability of cluster fusion, thereby stabilizing PG-Kmeans training. Consequently, in the HalfCheetah and Walker2d environments, increasing the initial \( k \) actually improves PG-Kmeans' clustering performance.
    }
    \label{fig: ablation study}
\end{figure*}

\begin{table}[ht]
\centering
\caption{Highest normalized evaluation returns of $k$ output policies, where the expert performance is scaled to 100.0. Every policy is evaluated on 10 different random seeds. Because we stopped the training once all the trajectories have got clear preference over some certain policy, so some networks were not fully trained for evaluation and those results are marked with *. We also highlighted all the results that clearly suffered from mode collapse in red. In these experiments, almost all trajectories were assigned to the same cluster.}
\label{table:pgkmeans_k_value_results}
\begin{tabular}{|l|ccccccccc|}
\toprule
\diagbox[width=6em]{Env}{$k$} & 3 & 4 & 5 & 6 & 7 & 8 & 10 & 12 & 15 \\
\midrule
HalfCheetah & \color{red}{40.47} & \color{red}{42.13} & 45.75* & 32.58* & 69.79 & 68.47 & \color{red}{28.76} & 88.95 & 49.59* \\
Hopper      & 44.42 & 45.56 & 46.47 & 111.24 & 80.84* & 69.85* & 93.44 & 24.07* & 101.33 \\
Ant         & \color{red}{22.27} & 116.91 & 111.94 & 61.97 & 103.52 & 130.80 & 47.78* & 67.38 & 67.80 \\
Walker2d    & 107.44 & 8.69* & 109.84 & \color{red}{62.74} & 103.32 & 1.29* & 98.26 & 6.33* & -0.09* \\
\bottomrule
\end{tabular}
\end{table}

\begin{figure}[H]
    \centering
    \begin{subfigure}{0.495\textwidth}
        \includegraphics[width=0.495\textwidth]{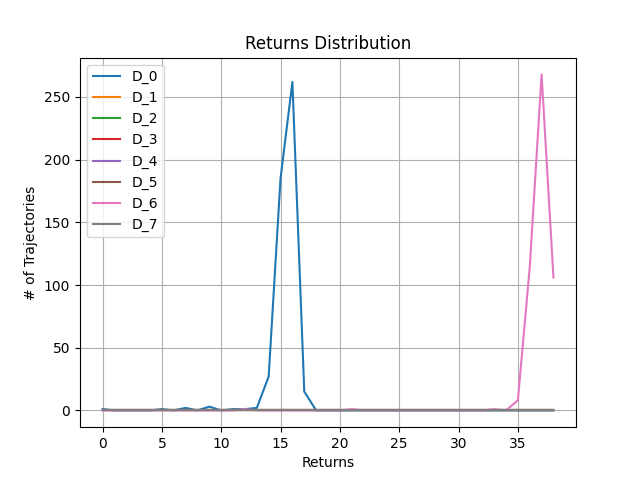}
        \includegraphics[width=0.495\textwidth]{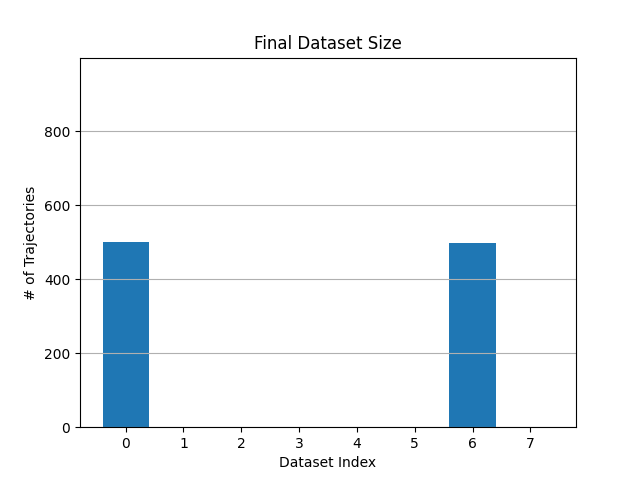}
        \caption{HalfCheetah with $k=8$}
    \end{subfigure}
    \begin{subfigure}{0.495\textwidth}
        \includegraphics[width=0.495\textwidth]{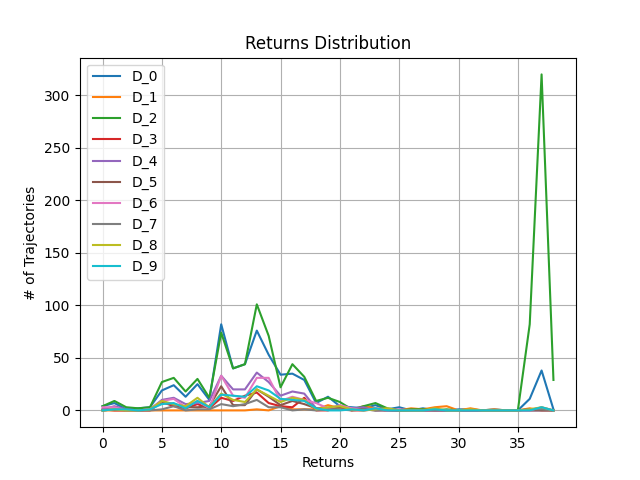}
        \includegraphics[width=0.495\textwidth]{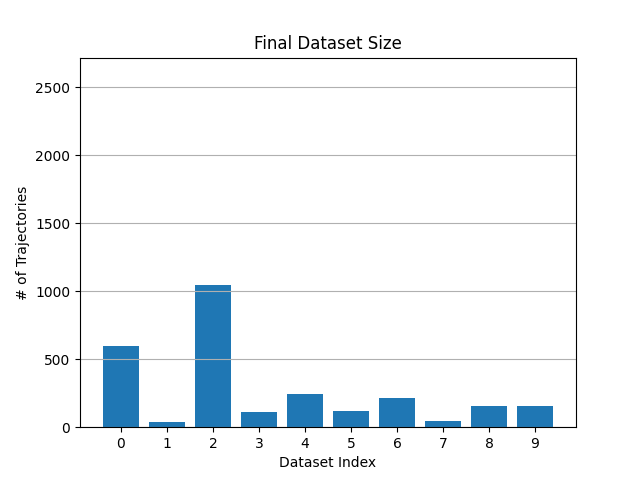}
        \caption{Hopper with $k=10$ }
    \end{subfigure}

    \vspace{1em}
    \begin{subfigure}{0.495\textwidth}
        \includegraphics[width=0.495\textwidth]{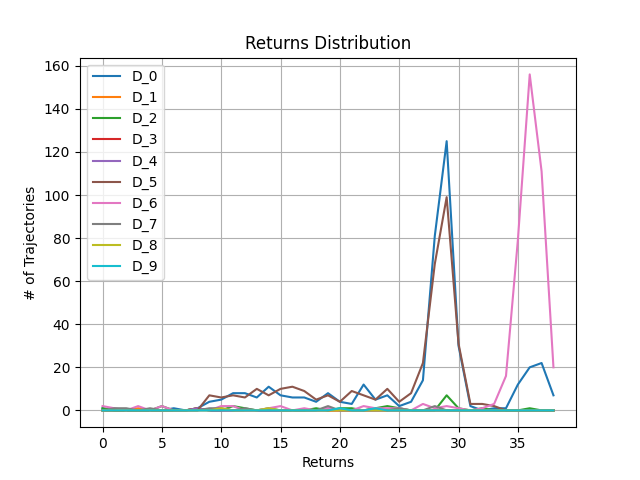}
        \includegraphics[width=0.495\textwidth]{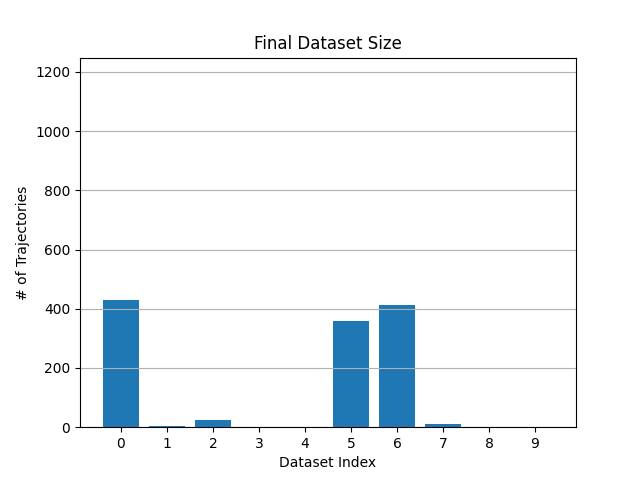}
        \caption{Ant with $k=10$ }
    \end{subfigure}
    \begin{subfigure}{0.495\textwidth}
        \includegraphics[width=0.495\textwidth]{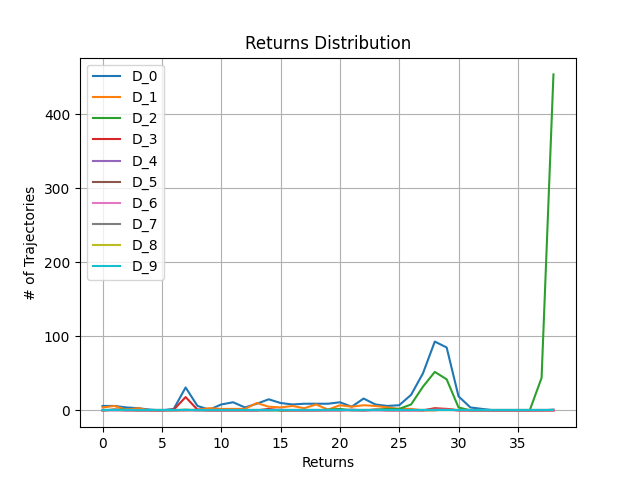}
        \includegraphics[width=0.495\textwidth]{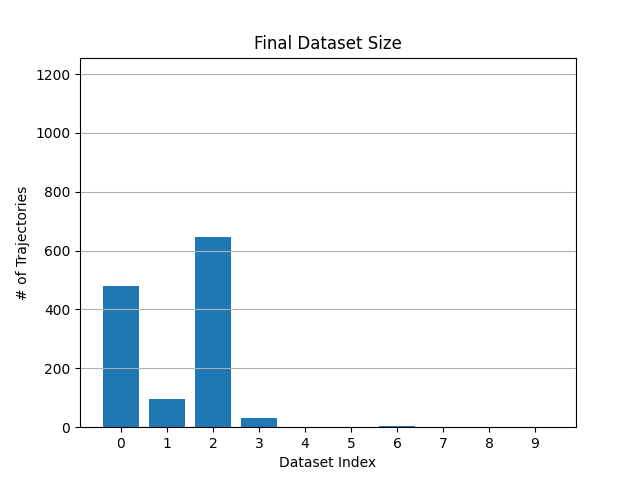}
        \caption{Walker2d with $k=10$ }
    \end{subfigure}

    \caption{Successful classification samples for different environments and $k$ values. From the figures, it is evident that the dataset is divided into 2–3 primary components. Since the dataset does not come with inherent classification labels, it is not possible to directly calculate classification accuracy. However, we can infer the differences between clusters indirectly by analyzing the return distribution of each cluster.}
    \label{fig:histogram}
\end{figure}